\documentclass[twoside]{article}
\usepackage[accepted]{aistats2e}
\usepackage{amsmath}
\usepackage{amssymb}
\usepackage{amsthm}
\usepackage{bm}
\usepackage{dsfont}
\usepackage{epsfig}
\usepackage{float}
\usepackage{graphics}
\usepackage{theapa}
\usepackage{times}

\newcommand{\commentout}[1]{}

\newtheorem{lemma}{Lemma}
\newtheorem{proposition}{Proposition}

\renewcommand{\qed}{\rule{5pt}{5pt}}

\newcommand{\bx}{{\bf x}}

\newcommand{\by}{{\bf y}}

\newcommand{\balpha}{{\bm \alpha}}
\newcommand{\bell}{{\bm \ell}}

\newcommand{\cH}{\mathcal{H}}
\newcommand{\cL}{\mathcal{L}}

\newcommand{\eps}{\varepsilon}

\newcommand{\realset}{\mathbb{R}}

\newcommand{\abs}[1]{\left|#1\right|}

\newcommand{\E}[2]{\mathbb{E}_{#1} \! \left[#2\right]}

\newcommand{\I}[1]{\mathds{1} \! \left\{#1\right\}}

\newcommand{\normw}[2]{\left\|#1\right\|_{#2}}

\newcommand{\set}[1]{\left\{#1\right\}}
\newcommand{\sgn}{\mathrm{sgn}}

\newcommand{\transpose}{^\mathsf{\scriptscriptstyle T}}

\begin{document}

\twocolumn[

\aistatstitle{Semi-Supervised Learning with Max-Margin Graph Cuts}

\aistatsauthor{Branislav Kveton \And Michal Valko \And Ali Rahimi \and Ling Huang}

\aistatsaddress{Intel Labs Santa Clara \And University of Pittsburgh \And Intel Labs Berkeley}]

\begin{abstract}
This paper proposes a novel algorithm for semi-supervised learning. This algorithm learns graph cuts that maximize the margin with respect to the labels induced by the harmonic function solution. We motivate the approach, compare it to existing work, and prove a bound on its generalization error. The quality of our solutions is evaluated on a synthetic problem and three UCI ML repository datasets. In most cases, we outperform manifold regularization of support vector machines, which is a state-of-the-art approach to semi-supervised max-margin learning.
\end{abstract}

\section{INTRODUCTION}
\label{sec:introduction}

Semi-supervised learning is a field of machine learning that studies learning from both labeled and unlabeled examples. This learning paradigm is suitable for real-world problems, where data is often abundant but the resources to label them are limited. As a result, many semi-supervised learning algorithms have been proposed in the past years \shortcite{zhu08semisupervised}. The closest to this work are semi-supervised support vector machines (S3VMs) \shortcite{bennett99semisupervised}, manifold regularization of support vector machines (SVMs) \shortcite{belkin06manifold}, and harmonic function solutions on data adjacency graphs \shortcite{zhu03semisupervised}. Manifold regularization of SVMs essentially combines the ideas of harmonic function solutions and semi-supervised SVMs in a single convex objective.

This paper proposes a different way of combining these two ideas. First, we compute the harmonic function solution on the data adjacency graph and then, we learn a discriminator, which is conditioned on the labels induced by this solution. We refer to our method as \emph{max-margin graph cuts} because the discriminator maximizes the margin with respect to the inferred labels. The method has many favorable properties. For instance, it incorporates the kernel trick \shortcite{wahba99support}, it takes advantage of sparse data adjacency matrices, and its generalization error can be bounded. Moreover, it typically yields better results than manifold regularization of SVMs, especially for linear and cubic decision boundaries.

In addition to proposing a new algorithm, this paper makes two contributions. First, we show how manifold regularization of linear and cubic SVMs fails on almost a trivial problem. Second, we show how to make the harmonic function solution with soft labeling constraints \shortcite{cortes08stability} stable.

The paper is organized as follows. In Section \ref{sec:reg HFS}, we review the harmonic function solution and discuss how to regularize it to interpolate between supervised learning on labeled examples and semi-supervised learning on all data. In Section \ref{sec:MMGC}, we introduce our learning algorithm. The algorithm is compared to existing work in Section \ref{sec:existing work} and we bound its generalization error in Section \ref{sec:theoretical analysis}. In Section \ref{sec:experiments}, we evaluate the quality of our solutions on UCI ML repository datasets, and show that they usually outperform manifold regularization of SVMs.

The following notation is used in the paper. The symbols $\bx_i$ and $y_i$ refer to the $i$-th data point and its label, respectively. The data points are divided into labeled and unlabeled sets, $l$ and $u$, and labels $y_i \! \in \! \set{-1, 1}$ are observed for the labeled data only. The cardinality of the labeled and unlabeled sets is $n_l = \abs{l}$ and $n_u = \abs{u}$, respectively, and the total number of training examples is $n = n_l + n_u$.

\section{REGULARIZED HARMONIC FUNCTION SOLUTION}
\label{sec:reg HFS}

In this section, we review the harmonic function solution of \shortciteA{zhu03semisupervised}. Moreover, we show how to regularize it to interpolate between semi-supervised learning on all data and supervised learning on labeled examples.

A standard approach to semi-supervised learning on graphs is to minimize the quadratic objective function:
\begin{align}
  \min_{\bell \in \realset^n} & \quad
  \bell\transpose L \bell \label{eq:HFS} \\
  \textrm{s.t.} & \quad
  \ell_i = y_i \textrm{ for all } i \in l; \nonumber
\end{align}
where $\bell$ denotes the vector of predictions, $L = D - W$ is the Laplacian of the data adjacency graph, which is represented by a matrix $W$ of pairwise similarities $w_{ij}$, and $D$ is a diagonal matrix whose entries are given by $d_i = \sum_j w_{ij}$. This problem has a closed-form solution:
\begin{align}
  \bell_u = (D_{uu} - W_{uu})^{-1} W_{ul} \bell_l,
  \label{eq:closed-form HFS}
\end{align}
which satisfies the \emph{harmonic property} $\ell_i = \frac{1}{d_i} \sum_{j \sim i} w_{ij} \ell_j$, and therefore is commonly known as the \emph{harmonic function solution}. Since the solution can be also computed as:
\begin{align}
  \bell_u = (I - P_{uu})^{-1} P_{ul} \bell_l,
  \label{eq:random walk HFS}
\end{align}
it can be viewed as a product of a random walk on the graph $W$ with the transition matrix $P = D^{-1} W$. The probability of moving between two arbitrary vertices $i$ and $j$ is $w_{ij} / d_i$, and the walk terminates when the reached vertex is labeled. Each element of the solution is given by:
\begin{align}
  \ell_i
  \ = & \ \ (I - P_{uu})_{iu}^{-1} P_{ul} \bell_l \nonumber \\
  \ = & \ \
  \underbrace{\sum_{j: y_j = 1} (I - P_{uu})_{iu}^{-1}
  P_{uj}}_{p_i^1} -
  \underbrace{\sum_{j: y_j = -1} (I - P_{uu})_{iu}^{-1}
  P_{uj}}_{p_i^{-1}} \nonumber \\
  \ = & \ \ p_i^1 - p_i^{-1},
  \label{eq:probability HFS}
\end{align}
where $p_i^1$ and $p_i^{-1}$ are probabilities by which the walk starting from the vertex $i$ ends at vertices with labels $1$ and $-1$, respectively. Therefore, when $\ell_i$ is rewritten as $\abs{\ell_i} \sgn(\ell_i)$, $\abs{\ell_i}$ can be interpreted as a \emph{confidence} of assigning the label $\sgn(\ell_i)$ to the vertex $i$. The maximum value of $\abs{\ell_i}$ is 1, and it is achieved when either $p_i^1 = 1$ or $p_i^{-1} = 1$. The closer the confidence $\abs{\ell_i}$ to 0, the closer the probabilities $p_i^1$ and $p_i^{-1}$ to 0.5, and the more \emph{uncertain} the label $\sgn(\ell_i)$.

To control the confidence of labeling unlabeled examples, we suggest regularizing the Laplacian $L$ as $L + \gamma_g I$, where $\gamma_g$ is a scalar and $I$ is the identity matrix. Similarly to our original problem (\ref{eq:HFS}), the corresponding harmonic function solution:
\begin{align}
  \min_{\bell \in \realset^n} & \quad
  \bell\transpose (L + \gamma_g I) \bell \label{eq:reg HFS} \\
  \textrm{s.t.} & \quad
  \ell_i = y_i \textrm{ for all } i \in l \nonumber
\end{align}
can be computed in a closed form:
\begin{align}
  \bell_u = (L_{uu} + \gamma_g I)^{-1} W_{ul} \bell_l.
  \label{eq:closed-form reg HFS}
\end{align}
It can be also interpreted as a random walk on the graph $W$ with an extra sink. At each step, this walk may terminate at the sink with probability $\gamma_g / (d_i + \gamma_g)$. Therefore, the scalar $\gamma_g$ essentially controls how the confidence $\abs{\ell_i}$ of labeling unlabeled vertices decreases with the number of hops from labeled vertices.

Several examples of how $\gamma_g$ affects the regularized solution are shown in Figure \ref{fig:HFS}. When $\gamma_g = 0$, the solution turns into the ordinary harmonic function solution. When $\gamma_g \! = \! \infty$, the confidence of labeling unlabeled vertices \mbox{decreases to zero.} Finally, note that our regularization corresponds to increasing all eigenvalues of the Laplacian $L$ by $\gamma_g$ \shortcite{smola03kernels}. In Section \ref{sec:theoretical analysis}, we use this property to bound the generalization error of our solutions.

\section{MAX-MARGIN GRAPH CUTS}
\label{sec:MMGC}

Our semi-supervised learning algorithm involves two steps. First, we obtain the regularized harmonic function solution $\bell^\ast$ (Equation \ref{eq:closed-form reg HFS}). The solution is computed from the system of linear equations $(L_{uu} + \gamma_g I) \bell_u = W_{ul} \bell_l$. This system of linear equations is sparse when the data adjacency graph $W$ is sparse. Second, we learn a max-margin discriminator, which is conditioned on the labels induced by the harmonic solution. The optimization problem is given by:
\begin{align}
  \min_{f \in \cH_K} & \quad \sum_{i : \abs{\ell_i^\ast} \geq \eps}
  \!\!\! V(f, \bx_i, \sgn(\ell_i^\ast)) + \gamma \normw{f}{K}^2
  \label{eq:MMGC} \\
  \textrm{s.t.} & \quad
  \bell^\ast = \arg\min_{\bell \in \realset^n}
  \bell\transpose (L + \gamma_g I) \bell \nonumber \\
  & \quad \textrm{s.t.} \ \ell_i = y_i \textrm{ for all } i \in l;
  \nonumber
\end{align}
where $V(f, \bx, y) = \max\{1 - y f(\bx), 0\}$ denotes the \emph{hinge loss}, $f$ is a function from some \emph{reproducing kernel Hilbert space (RKHS)} $\cH_K$, and $\normw{\cdot}{K}$ is the norm that measures the complexity of $f$.

Training examples $\bx_i$ in our problem are selected based on our confidence into their labels. When the labels are highly \emph{uncertain}, which means that $\abs{\ell_i^\ast} < \eps$ for some small $\eps \geq 0$, the examples are excluded from learning. Note that as the regularizer $\gamma_g$ increases, the values $\abs{\ell_i^\ast}$ decrease towards 0 (Figure \ref{fig:HFS}), and the $\eps$ thresholding allows for smooth interpolations between supervised learning on labeled examples and semi-supervised learning on all data. The tradeoff between the regularization of $f$ and the minimization of hinge losses $V(f, \bx_i, \sgn(\ell_i^\ast))$ is controlled by the parameter $\gamma$.

Due to the representer theorem \shortcite{wahba99support}, the optimal solution $f^\ast$ to our problem has a special form:
\begin{align}
  f^\ast(\bx) =
  \sum_{i : \abs{\ell_i^\ast} \geq \eps} \alpha_i^\ast k(\bx_i, \bx),
  \label{eq:MMGC representer}
\end{align}
where $k(\cdot, \cdot)$ is a Mercer kernel associated with the RKHS $\cH_K$. Therefore, we can apply the kernel trick and optimize rich classes of discriminators in a finite-dimensional space of $\balpha = (\alpha_1, \dots, \alpha_n)$. Finally, note that when $\gamma_g = \infty$, our solution $f^\ast$ corresponds to supervised learning with SVMs.

\begin{figure*}[t]
  \centering
  \includegraphics[width=4.8in, bb=1.25in 0in 7.25in 1.5in]{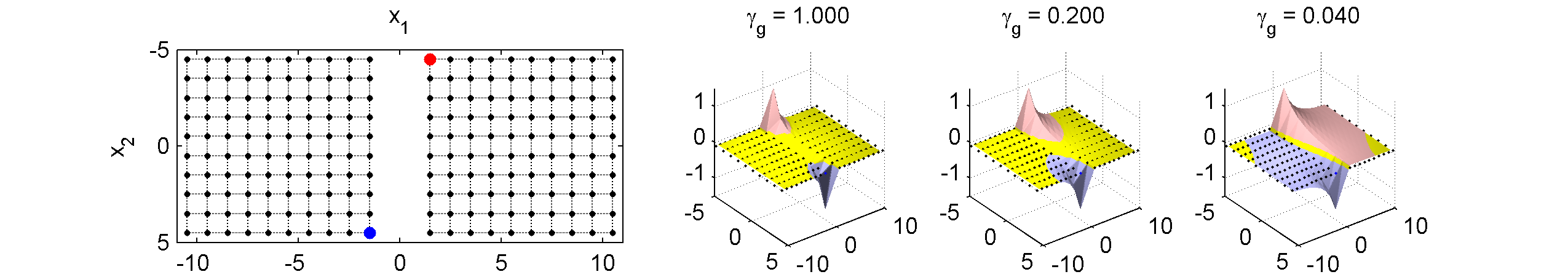}
  \vspace{0.1in} \\
  \hspace{-0.6in} \textbf{(a)} \hspace{2.55in} \textbf{(b)}
  \caption{\textbf{a.} An example of a simple data adjacency graph. The vertices of the graph are depicted as dots. The red and blue dots are labeled vertices. The edges of the graph are shown as dotted lines and weighted as $w_{ij} = \exp[- \normw{\bx_i - \bx_j}{2}^2 / 2]$. \textbf{b.} Three regularized harmonic function solutions on the data adjacency graph from Figure \ref{fig:HFS}a. The plots are cubic interpolations of the solutions. The pink and blue colors denote parts of the feature space $\bx$ where $\ell_i > 0$ and $\ell_i < 0$, respectively. The yellow color marks regions where the confidence $\abs{\ell_i}$ is less than 0.05.}
  \label{fig:HFS}
\end{figure*}

\section{EXISTING WORK}
\label{sec:existing work}

Most of the existing work on semi-supervised max-margin learning can be viewed as manifold regularization of SVMs \shortcite{belkin06manifold} or semi-supervised SVMs with the hat loss on unlabeled data \shortcite{bennett99semisupervised}. The two approaches are reviewed in the rest of the section.

\subsection{SEMI-SUPERVISED SVMS}
\label{sec:S3VM}

Semi-supervised support vector machines with the \emph{hat loss} $\widehat{V}(f, \bx) = \max\{1 - \abs{f(\bx)}, 0\}$ on unlabeled data \shortcite{bennett99semisupervised}:
\begin{align}
  \min_f \ \sum_{i \in l} V(f, \bx_i, y_i) +
  \gamma \normw{f}{K}^2 +
  \gamma_u \sum_{i \in u} \widehat{V}(f, \bx_i)
  \label{eq:S3VM}
\end{align}
compute max-margin decision boundaries that avoid dense regions of data. The hat loss makes the optimization problem non-convex. As a result, it is hard to solve the problem optimally and most of the work in this field has focused on approximations. A comprehensive review of these methods was done by \shortciteA{zhu08semisupervised}.

In comparison to semi-supervised SVMs, learning of max-margin graph cuts (\ref{eq:MMGC}) is a convex problem. The convexity is achieved by having a two-stage learning algorithm. First, we infer labels of unlabeled examples using the regularized harmonic function solution, and then, we minimize the corresponding convex losses.

\subsection{MANIFOLD REGULARIZATION OF SVMS}
\label{sec:MR}

Manifold regularization of SVMs \shortcite{belkin06manifold}:
\begin{align}
  \min_{f \in \cH_K} \ \sum_{i \in l} V(f, \bx_i, y_i) +
  \gamma \normw{f}{K}^2 +
  \gamma_u {\bf f}\transpose L {\bf f},
  \label{eq:MR}
\end{align}
where ${\bf f} = (f(\bx_1), \dots, f(\bx_n))$, computes max-margin decision boundaries that are smooth in the feature space. The smoothness is achieved by the minimization of the regularization term ${\bf f}\transpose L {\bf f}$. Intuitively, when two examples are close on a manifold, the minimization of ${\bf f}\transpose L {\bf f}$ leads to assigning the same label to both examples.

In some aspects, manifold regularization is similar to max-margin graph cuts. In particular, note that its objective (\ref{eq:MR}) is similar to the regularized harmonic function solution (\ref{eq:reg HFS}). Both objectives involve regularization by a manifold, ${\bf f}\transpose L {\bf f}$ and $\bell\transpose L \bell$, regularization in the space of learned parameters, $\normw{f}{K}^2$ and $\bell\transpose I \bell$, and some labeling constraints $V(f, \bx_i, y_i)$ and $\ell_i = y_i$. Since max-margin graph cuts are learned conditionally on the harmonic function solution, the problems (\ref{eq:MMGC}) and (\ref{eq:MR}) may sometimes have similar solutions. A necessary condition is that the regularization terms in both objectives are weighted in the same proportions, for instance, by setting $\gamma_g = \gamma / \gamma_u$. We adopt this setting when manifold regularization of SVMs is compared to max-margin graph cuts in Section \ref{sec:experiments}.

\subsection{MANIFOLD REGULARIZATION FAILS}
\label{sec:MR fails}

The major difference between manifold regularization (\ref{eq:MR}) and the regularized harmonic function solution (\ref{eq:reg HFS}) is in the space of optimized parameters. In particular, manifold regularization is performed on a class of functions $\cH_K$. When this class is severely restricted, such as linear functions, the minimization of ${\bf f}\transpose L {\bf f}$ may lead to results, which are significantly worse than the harmonic function solution.

This issue can be illustrated on the problem from Figure \ref{fig:HFS}, where we learn a linear decision boundary $f(\bx) = \alpha_1 x_1 + \alpha_2 x_2$ through manifold regularization of linear SVMs:
\begin{align}
  \min_{\alpha_1, \alpha_2} \ \sum_{i \in l} V(f, \bx_i, y_i) +
  \gamma [\alpha_1^2 + \alpha_2^2] +
  \gamma_u {\bf f}\transpose L {\bf f}.
  \label{eq:linear MR}
\end{align}
The structure of our problem simplifies the computation of the regularization term ${\bf f}\transpose L {\bf f}$. In particular, since all edges in the data adjacency graph are either horizontal or vertical, the term ${\bf f}\transpose L {\bf f}$ can be expressed as a function of $\alpha_1^2$ and $\alpha_2^2$:
\begin{align}
  {\bf f}\transpose L {\bf f}
  \ = & \ \ \frac{1}{2} \sum_{i, j} w_{ij} (f(\bx_i) - f(\bx_j))^2
  \nonumber \\
  \ = & \ \ \frac{1}{2} \sum_{i, j} w_{ij}
  (\alpha_1 (\bx_{i1} - \bx_{j1}) + \alpha_2 (\bx_{i2} - \bx_{j2}))^2
  \nonumber \\
  \ = & \ \ \frac{\alpha_1^2}{2} \underbrace{\sum_{i, j} w_{ij}
  (\bx_{i1} - \bx_{j1})^2}_{\Delta = 218.351} + \nonumber \\
  & \ \ \frac{\alpha_2^2}{2} \underbrace{\sum_{i, j} w_{ij}
  (\bx_{i2} - \bx_{j2})^2}_{\Delta = 218.351},
  \label{eq:linear manifold}
\end{align}
and incorporated in our objective function as an additional weight at the regularizer $[\alpha_1^2 + \alpha_2^2]$:
\begin{align}
  \min_{\alpha_1, \alpha_2} \ \sum_{i \in l} V(f, \bx_i, y_i) +
  \left(\gamma + \frac{\gamma_u \Delta}{2}\right)
  [\alpha_1^2 + \alpha_2^2].
  \label{eq:linear SVM}
\end{align}
Thus, manifold regularization of linear SVMs on our problem can be viewed as supervised learning with linear SVMs with a varying weight at the regularizer. Since the problem involves only two labeled examples, changes in the weight $\left(\gamma + \frac{\gamma_u \Delta}{2}\right)$ do not affect the direction of the discriminator $f^\ast(\bx) = 0$ and only change the slope of $f^\ast$ (Figure \ref{fig:synthetic cuts}).

The above analysis shows that the discriminator $f^\ast(\bx) = 0$ does not change with $\gamma_u$. As a result, all discriminators are equal to the discriminator for $\gamma_u = 0$, which can be learned by linear SVMs, and none of them solves our problem optimally. Max-margin graph cuts solve the problem optimally for small values of $\gamma_g$ (Figure \ref{fig:synthetic cuts}).

A similar line of reasoning can be used to extend our results to polynomial kernels. Figure \ref{fig:synthetic cuts} indicates that max-margin learning with the cubic kernel exhibits similar trends to the linear case.

\begin{figure*}
  \centering
  \includegraphics[width=4.8in, bb=1.25in 0.75in 7.25in 9.75in]{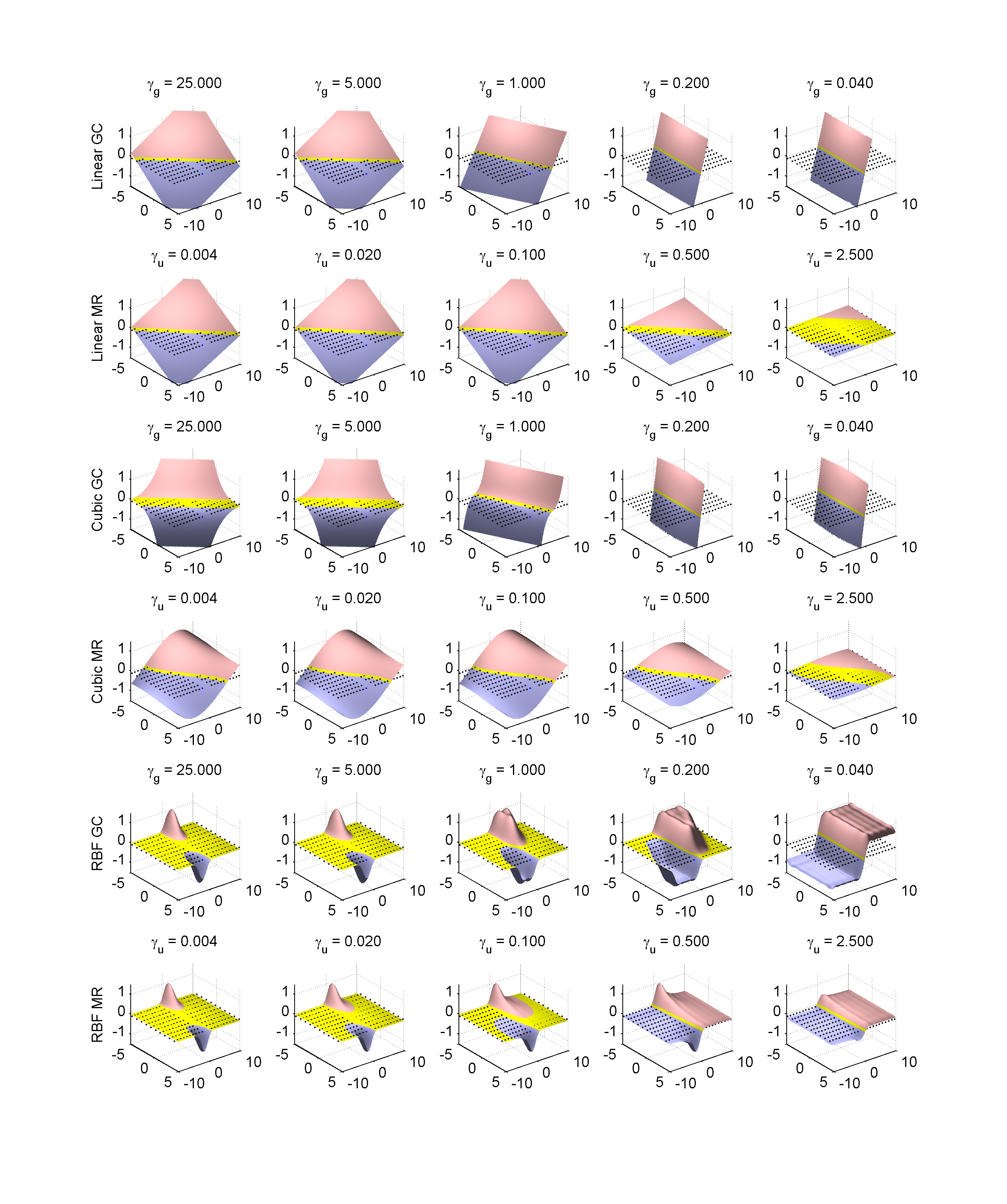}
  \caption{Linear, cubic, and RBF decision boundaries obtained by manifold regularization of SVMs (MR) and max-margin graph cuts (GC) on the problem from Figure \ref{fig:HFS}. The regularization parameter $\gamma_g = \gamma / \gamma_u$ is set as suggested in Section \ref{sec:MR}, $\gamma \! = \! 0.1$, and $\eps \! = \! 0.01$. The pink and blue colors denote parts of the feature space $\bx$ where the discriminators $f$ are positive and negative, respectively. The yellow color marks regions where $\abs{f(\bx)} < 0.05$.}
  \label{fig:synthetic cuts}
\end{figure*}

\section{THEORETICAL ANALYSIS}
\label{sec:theoretical analysis}

The notion of algorithmic stability can be used to bound the generalization error of many learning algorithms \shortcite{bousquet02stability}. In this section, we discuss how to make the harmonic function solution stable and prove a bound on the generalization error of max-margin cuts (\ref{eq:MMGC}). Our bound combines existing transductive \shortcite{belkin04regularization,cortes08stability} and inductive \shortcite{vapnik95nature} bounds.

\subsection{GENERALIZATION ERROR}
\label{sec:generalization error}

Our objective is to show that the \emph{risk} of our solutions $f$:
\begin{align}
  R_P(f) = \E{P(\bx)}{\cL(f(\bx), y(\bx))}
  \label{eq:risk}
\end{align}
is bounded by the \emph{empirical risk} on graph-induced labels:
\begin{align}
  \frac{1}{n} \sum_i \cL(f(\bx_i), \sgn(\ell_i^\ast))
  \label{eq:empirical risk}
\end{align}
and error terms, which can be computed from training data. The function $\cL(y', y) \! = \! \I{\sgn(y') \! \neq \! y}$ computes the zero-one loss of the prediction $\sgn(y')$ given the ground truth $y$, and $P(\bx)$ is the distribution of our data. For simplicity, we assume that the label $y$ is a deterministic function of $\bx$. Our proof starts by relating $R_P(f)$ and graph-induced labels $\ell_i^\ast$.

\begin{lemma}
\label{lem:inductive bound} Let $f$ be from a function class with the VC dimension $h$ and $\bx_i$ be $n$ examples, which are sampled i.i.d. with respect to the distribution $P(\bx)$. Then the inequality:
\begin{align*}
  R_P(f)
  \ \leq & \ \
  \frac{1}{n} \sum_i \cL(f(\bx_i), \sgn(\ell_i^\ast)) \ + \\
  & \ \
  \frac{1}{n} \sum_i (\ell_i^\ast - y_i)^2 + \\
  & \ \
  \underbrace{\sqrt{\frac{h (\ln(2 n / h) + 1) -
  \ln(\eta / 4)}{n}}}_{\emph{inductive error } \Delta_I(h, n, \eta)}
\end{align*}
holds with probability $1 - \eta$, where $y_i$ and $\ell_i^\ast$ represent the true and graph-induced soft labels, respectively.
\end{lemma}
{\bf Proof:} Based on Equations 3.15 and 3.24 \shortcite{vapnik95nature}, the inequality:
\begin{align*}
  R_P(f) \leq \frac{1}{n} \sum_i \cL(f(\bx_i), y_i) +
  \Delta_I(h, n, \eta)
\end{align*}
holds with probability $1 - \eta$. Our final claim follows from bounding all terms $\cL(f(\bx_i), y_i)$ as:
\begin{align*}
  \cL(f(\bx_i), y_i) \leq \cL(f(\bx_i), \sgn(\ell_i^\ast)) +
  (\ell_i^\ast - y_i)^2.
\end{align*}
The above bound holds for any $y_i \in \set{-1, 1}$ and $\ell_i^\ast$. \qed

\bigskip It is hard to bound the error term $\frac{1}{n} \sum_i (\ell_i^\ast - y_i)^2$ when the constraints $\ell_i = y_i$ (\ref{eq:reg HFS}) are enforced in a hard manner. Thus, in the rest of our analysis, we consider a relaxed version of the harmonic function solution \shortcite{cortes08stability}:
\begin{align}
  \min_{\bell \in \realset^n} \
  (\bell - \by)\transpose C (\bell - \by) + \bell\transpose L \bell,
  \label{eq:soft HFS}
\end{align}
where $L$ is the Laplacian of the data adjacency graph, $C$ is a diagonal matrix such that $C_{ii} \! = \! c_l$ for all labeled examples, and $C_{ii} = c_u$ otherwise, and $\by$ is a vector of pseudo-targets such that $y_i$ is the label of the $i$-th example when the example is labeled, and $y_i = 0$ otherwise.

The generalization error of the solution to the problem (\ref{eq:soft HFS}) is bounded in Lemma \ref{lem:transductive bound}. To simplify the proof, we assume that $c_l = 1$ and $c_l > c_u$.

\begin{lemma}
\label{lem:transductive bound} Let $\bell^\ast$ be a solution to the problem:
\begin{align*}
  \min_{\bell \in \realset^n} \
  (\bell - \by)\transpose C (\bell - \by) +
  \bell\transpose Q \bell,
\end{align*}
where $Q = L + \gamma_g I$ and all labeled examples $l$ are selected i.i.d. Then the inequality:
\begin{align*}
  R_P^{\scriptscriptstyle{W}}(\bell^\ast) \ \leq & \ \
  \widehat{R}_P^{\scriptscriptstyle{W}}(\bell^\ast) +
  \underbrace{\beta + \sqrt{\frac{2 \ln(2 / \delta)}{n_l}}
  (n_l \beta + 4)}_{\emph{transductive error }
  \Delta_T(\beta, n_l, \delta)} \\
  \beta \ \leq & \ \
  2 \left[\frac{\sqrt{2}}{\gamma_g + 1} +
  \sqrt{2 n_l} \frac{1 - \sqrt{c_u}}{\sqrt{c_u}}
  \frac{\lambda_M(L) + \gamma_g}{\gamma_g^2 + 1}\right]
\end{align*}
holds with probability $1 - \delta$, where:
\begin{align*}
  R_P^{\scriptscriptstyle{W}}(\bell^\ast)
  \ = & \ \
  \frac{1}{n} \sum_i (\ell_i^\ast - y_i)^2 \\
  \widehat{R}_P^{\scriptscriptstyle{W}}(\bell^\ast)
  \ = & \ \
  \frac{1}{n_l} \sum_{i \in l} (\ell_i^\ast - y_i)^2
\end{align*}
are risk terms for all and labeled vertices, respectively, and $\beta$ is the stability coefficient of the solution $\bell^\ast$.
\end{lemma}
{\bf Proof:} Our risk bound follows from combining Theorem 1 of \shortciteA{belkin04regularization} with the assumptions $\abs{y_i} \leq 1$ and $\abs{\ell_i^\ast} \leq 1$. The coefficient $\beta$ is derived based on Section 5 of \shortciteA{cortes08stability}. In particular, based on the properties of the matrix $C$ and Proposition 1 \shortcite{cortes08stability}, we conclude:
\begin{align*}
  \beta = 2 \left[\frac{\sqrt{2}}{\lambda_m(Q) + 1} +
  \sqrt{2 n_l} \frac{1 - \sqrt{c_u}}{\sqrt{c_u}}
  \frac{\lambda_M(Q)}{(\lambda_m(Q) + 1)^2}\right],
\end{align*}
where $\lambda_m(Q)$ and $\lambda_M(Q)$ refer to the smallest and largest eigenvalues of $Q$, respectively, and can be further rewritten as $\lambda_m(Q) = \lambda_m(L) + \gamma_g$ and $\lambda_M(Q) = \lambda_M(L) + \gamma_g$. Our final claim directly follows from applying the lower bounds $\lambda_m(L) \geq 0$ and $(\lambda_m(L) + \gamma_g + 1)^2 \geq \gamma_g^2 + 1$. \qed

\bigskip Lemma \ref{lem:transductive bound} is practical when the error $\Delta_T(\beta, n_l, \delta)$ decreases at the rate of $O(n_l^{- \frac{1}{2}})$. This is achieved when $\beta \! = \! O(1 / n_l)$, which corresponds to $\gamma_g \! = \! \Omega(n_l^\frac{3}{2})$. Thus, when the problem (\ref{eq:soft HFS}) is sufficiently regularized, its solution is stable, and the generalization error of the solution is bounded.

Lemmas \ref{lem:inductive bound} and \ref{lem:transductive bound} can be combined using the union bound.

\begin{proposition}
\label{prop:MMGC bound} Let $f$ be from a function class with the VC dimension $h$. Then the inequality:
\begin{align*}
  R_P(f)
  \ \leq & \ \
  \frac{1}{n} \sum_i \cL(f(\bx_i), \sgn(\ell_i^\ast)) \ + \\
  & \ \
  \widehat{R}_P^{\scriptscriptstyle{W}}(\bell^\ast) +
  \Delta_T(\beta, n_l, \delta) + \Delta_I(h, n, \eta)
\end{align*}
holds with probability $1 - (\eta + \delta)$.
\end{proposition}

\bigskip The above result can be viewed as follows. If both $n$ and $n_l$ are large, the sum of $\frac{1}{n} \sum_i \cL(f(\bx_i), \sgn(\ell_i^\ast))$ and $\widehat{R}_P^{\scriptscriptstyle{W}}(\bell^\ast)$ provides a good estimate of the risk $R_P(f)$. Unfortunately, our bound is not practical for setting $\gamma_g$ because it is hard to find $\gamma_g$ that minimizes both $\widehat{R}_P^{\scriptscriptstyle{W}}(\bell^\ast)$ and $\Delta_T(\beta, n_l, \delta)$. The same phenomenon was observed by \shortciteA{belkin04regularization} in a similar context. To solve our problem, we suggest setting $\gamma_g$ based on the validation set. This methodology is used in the experimental section.

\subsection{THRESHOLD $\eps$}
\label{sec:threshold eps}

Finally, note that when $\abs{\ell_i^\ast} < \eps$, where $\eps$ is a small number, $\abs{\ell_i^\ast - y_i}$ is close to 1 irrespective of $y_i$, and a trivial upper bound $\cL(f(\bx_i), y_i) \! \leq \! 1$ is almost as good as $\cL(f(\bx_i), y_i) \! \leq \! \cL(f(\bx_i), \sgn(\ell_i^\ast)) + (\ell_i^\ast - y_i)^2$ for any $f$. This allows us to justify the $\eps$ threshold in the problem (\ref{eq:MMGC}). In particular, note that $\cL(f(\bx_i), y_i)$ is bounded by $1 - (\ell_i^\ast - y_i)^2 + (\ell_i^\ast - y_i)^2$. When $\abs{\ell_i^\ast} < \eps$, $1 - (\ell_i^\ast - y_i)^2 < 2 \eps - \eps^2$, and we conclude the following.

\begin{figure*}[t]
  \centering
  \includegraphics[width=4.8in, bb=1.25in 3.5in 7.25in 7.5in]{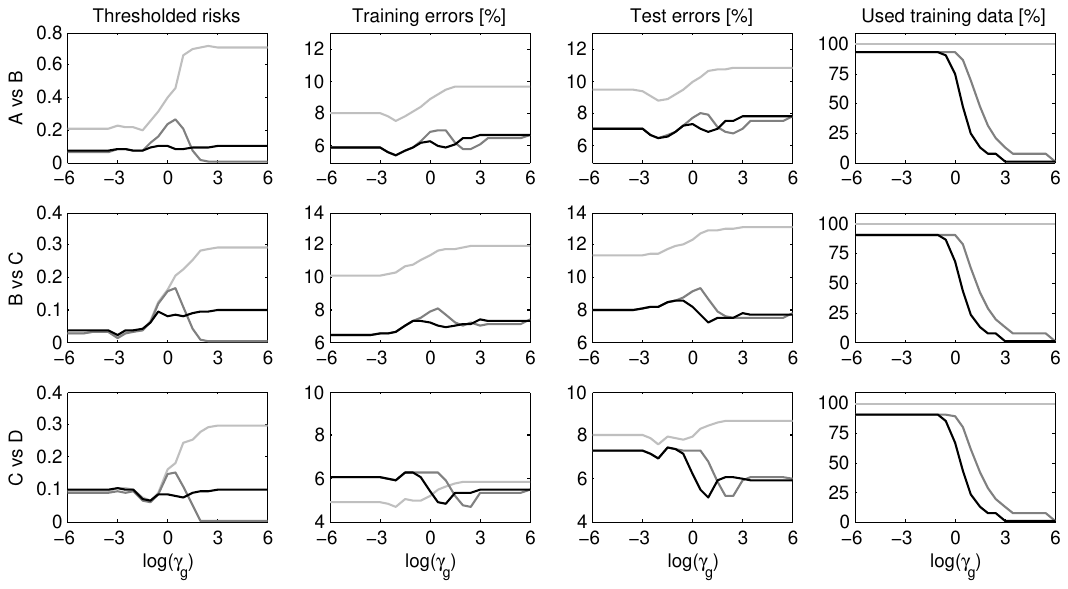}
  \caption{The thresholded empirical risk $\frac{1}{n} \sum_{i : \abs{\ell_i^\ast} \geq \eps} \cL(f^\ast(\bx_i), \sgn(\ell_i^\ast)) + \frac{2 \eps n_\eps}{n}$ of the optimal max-margin graph cut $f^\ast$ (\ref{eq:MMGC}), its training and test errors, and the percentage of training examples such that $\abs{\ell_i^\ast} \geq \eps$, on 3 letter recognition problems from the UCI ML repository. The plots are shown as functions of the parameter $\gamma_g$, and correspond to the thresholds $\eps = 0$ (light gray lines), $\eps = 10^{-6}$ (dark gray lines), and $\eps = 10^{-3}$ (black lines). All results are averaged over 50 random choices of 1 percent of labeled examples.}
  \label{fig:thresholded objective}
\end{figure*}

\begin{proposition}
\label{prop:MMGC thresholded bound} Let $f$ be from a function class with the VC dimension $h$ and $n_\eps$ be the number of examples such that $\abs{\ell_i^\ast} < \eps$. Then the inequality:
\begin{align*}
  R_P(f)
  \ \leq & \ \
  \frac{1}{n} \!\!\! \sum_{i : \abs{\ell_i^\ast} \geq \eps}
  \!\!\! \cL(f(\bx_i), \sgn(\ell_i^\ast)) +
  \frac{2 \eps n_\eps}{n} \ + \\
  & \ \
  \widehat{R}_P^{\scriptscriptstyle{W}}(\bell^\ast) +
  \Delta_T(\beta, n_l, \delta) + \Delta_I(h, n, \eta)
\end{align*}
holds with probability $1 - (\eta + \delta)$.
\end{proposition}
{\bf Proof:} The generalization bound is proved as:
\begin{align*}
  R_P(f)
  \ \leq & \ \ \widehat{R}_P(f) + \Delta_I(h, n, \eta) \\
  \ = & \ \
  \frac{1}{n} \!\!\! \sum_{i : \abs{\ell_i^\ast} \geq \eps}
  \!\!\! \cL(f(\bx_i), y_i) +
  \frac{1}{n} \!\!\! \sum_{i : \abs{\ell_i^\ast} < \eps}
  \!\!\! \cL(f(\bx_i), y_i) \ + \\
  & \ \ \Delta_I(h, n, \eta) \\
  \ \leq & \ \
  \frac{1}{n} \!\!\! \sum_{i : \abs{\ell_i^\ast} \geq \eps}
  \!\!\! \left[\cL(f(\bx_i), \sgn(\ell_i^\ast)) +
  (\ell_i^\ast - y_i)^2\right] + \\
  & \ \
  \frac{1}{n} \!\!\! \sum_{i : \abs{\ell_i^\ast} < \eps}
  \!\!\! \left[1 - (\ell_i^\ast - y_i)^2 +
  (\ell_i^\ast - y_i)^2\right] + \\
  & \ \ \Delta_I(h, n, \eta) \\
  \ = & \ \
  \frac{1}{n} \!\!\! \sum_{i : \abs{\ell_i^\ast} \geq \eps}
  \!\!\! \cL(f(\bx_i), \sgn(\ell_i^\ast)) \ + \\
  & \ \
  \frac{1}{n} \!\!\! \sum_{i : \abs{\ell_i^\ast} < \eps}
  \!\!\! \left[1 - (\ell_i^\ast - y_i)^2\right] +
  \frac{1}{n} \sum_i (\ell_i^\ast - y_i)^2 + \\
  & \ \ \Delta_I(h, n, \eta) \\
  \ \leq & \ \
  \frac{1}{n} \!\!\! \sum_{i : \abs{\ell_i^\ast} \geq \eps}
  \!\!\! \cL(f(\bx_i), \sgn(\ell_i^\ast)) +
  \frac{2 \eps n_\eps}{n} \ + \\
  & \ \
  \widehat{R}_P^{\scriptscriptstyle{W}}(\bell^\ast) +
  \Delta_T(\beta, n_l, \delta) + \Delta_I(h, n, \eta).
\end{align*}
The last step follows from the inequality $1 - (\ell_i^\ast - y_i)^2 < 2 \eps$ and Lemma \ref{lem:transductive bound}. \qed

\bigskip When $\eps \leq n_l^{- \frac{1}{2}}$, the new upper bound is asymptotically as good as the bound in Proposition \ref{prop:MMGC bound}. As a result, we get the same convergence guarantees although highly-uncertain labels $\abs{\ell_i^\ast} < \eps$ are excluded from our optimization.

In practice, optimization of the thresholded objective often yields a lower risk $\frac{1}{n} \sum_{i : \abs{\ell_i^\ast} \geq \eps} \cL(f^\ast(\bx_i), \sgn(\ell_i^\ast)) + \frac{2 \eps n_\eps}{n}$, and also lower training and test errors. This is a result of excluding the most uncertain examples \mbox{$\abs{\ell_i^\ast} \! < \! \eps$ from learning.} Figure \ref{fig:thresholded objective} illustrates these trends on three learning problems. Note that the parameters $\gamma_g$ and $\eps$ are redundant in the sense that the same result is often achieved by different combinations of parameter values. This problem is addressed in the experimental section by fixing $\eps$ and optimizing $\gamma_g$ only.

\section{EXPERIMENTS}
\label{sec:experiments}

The experimental section is divided into two \mbox{parts. The first} part compares max-margin graph cuts to manifold regularization of SVMs on the problem from Figure \ref{fig:HFS}. The second part compares max-margin graph cuts, manifold regularization of SVMs, and supervised learning with SVMs on three UCI ML repository datasets \shortcite{ucimlrepository}.

Manifold regularization of SVMs is evaluated based on the implementation of \shortciteA{belkin06manifold}. Max-margin graph cuts and SVMs are implemented using LIBSVM \shortcite{chang01libsvm}.

\subsection{SYNTHETIC PROBLEM}
\label{sec:synthetic experiments}

The first experiment (Figure \ref{fig:synthetic cuts}) illustrates linear, cubic, and RBF graph cuts (\ref{eq:MMGC}) on the synthetic problem from Figure \ref{fig:HFS}. The cuts are shown for various settings of the regularization parameter $\gamma_g$. As $\gamma_g$ decreases, note that the cuts gradually interpolate between supervised learning on just two labeled examples and semi-supervised learning on all data. The resulting discriminators are max-margin decision boundaries that separate the corresponding colored \mbox{regions in Figure \ref{fig:HFS}.}

Figure \ref{fig:synthetic cuts} also shows that manifold regularization of SVMs (\ref{eq:MR}) with linear and cubic kernels cannot perfectly separate the two clusters in Figure \ref{fig:HFS} for any setting of the parameter $\gamma_u$. The reason for this problem is discussed in Section \ref{sec:MR fails}. Finally, note the similarity between max-margin graph cuts and manifold regularization of SVMs with the RBF kernel. This similarity was suggested in Section \ref{sec:MR}.

\subsection{UCI ML REPOSITORY DATASETS}
\label{sec:UCI ML repository experiments}

\begin{figure*}[t]
  \centering
  {\small
  \begin{tabular}{l r|r r r|r r r|r r r} \hline \hline
    & & \multicolumn{9}{|c}{Misclassification errors [\%]} \\ \cline{3-11}
    Dataset & $L$ & \multicolumn{3}{|c|}{Linear kernel} &
    \multicolumn{3}{|c|}{Cubic kernel} & \multicolumn{3}{|c}{RBF kernel} \\
    & & SVM & MR & GC & SVM & MR & GC & SVM & MR & GC \\ \hline
    &  1 & 18.90 & 30.94 & {\bf 15.79} & 20.54 & 25.96 & {\bf 17.45} & 20.06 & 17.61 & {\bf 16.01} \\
    Letter
    &  2 & 12.92 & 28.45 & {\bf 10.79} & 12.18 & 18.34 & {\bf 10.90} & 13.52 & 13.10 & {\bf 11.83} \\
    recognition
    &  5 &  8.21 & 27.13 & {\bf  5.65} &  5.49 & 18.77 & {\bf  4.80} &  6.81 &  8.06 & {\bf  5.65} \\
    & 10 &  6.51 & 25.45 & {\bf  3.96} &  4.17 & 14.03 & {\bf  2.96} &  4.95 &  6.14 & {\bf  3.32} \\ \hline
    &  1 &  7.06 &  9.59 & {\bf  6.88} &  9.62 & {\bf  5.29} &  8.55 &  8.22 & {\bf  6.36} &  7.65 \\
    Digit
    &  2 &  4.87 &  7.97 & {\bf  4.60} &  6.06 & {\bf  5.06} &  5.09 &  6.17 & {\bf  4.21} &  5.61 \\
    recognition
    &  5 &  2.97 &  3.68 & {\bf  2.29} &  3.04 & {\bf  2.27} &  2.36 &  2.74 &  2.29 & {\bf  2.19} \\
    & 10 &  1.70 &  2.86 & {\bf  1.59} &  1.87 & {\bf  1.60} &  1.74 &  1.68 &  1.75 & {\bf  1.35} \\ \hline
    &  1 & 14.02 & 11.81 & {\bf 10.27} & 23.30 & {\bf 12.02} & 14.10 & 14.02 & 11.60 & {\bf  9.51} \\
    Image
    &  2 &  8.54 & 10.87 & {\bf  7.69} & 14.28 & 13.07 & {\bf  7.73} &  9.06 &  8.93 & {\bf  7.34} \\
    segmentation
    &  5 &  4.73 &  7.83 & {\bf  4.49} &  8.32 &  8.79 & {\bf  7.17} &  5.87 &  5.43 & {\bf  5.31} \\
    & 10 &  3.30 &  6.26 & {\bf  3.28} &  3.65 &  6.64 & {\bf  3.60} &  3.84 &  4.81 & {\bf  3.73} \\
    \hline \hline
  \end{tabular}
  }
  \caption{Comparison of SVMs, max-margin graph cuts (GC), and manifold regularization of SVMs (MR) on three datasets from the UCI ML repository. The fraction of labeled examples $L$ varies from 1 to 10 percent.}
  \label{fig:UCI ML repository cuts}
\end{figure*}

The second experiment (Figure \ref{fig:UCI ML repository cuts}) shows that max-margin graph cuts (\ref{eq:MMGC}) typically outperform manifold regularization of SVMs (\ref{eq:MR}) and supervised learning with SVMs. The experiment is done on three UCI ML repository datasets: letter recognition, digit recognition, and image segmentation. The datasets are multi-class and thus, we transform each of them into a set of binary classification problems. The digit recognition and image segmentation datasets are converted into 45 and 15 problems, respectively, where all classes are discriminated against every other class. The letter recognition dataset is turned into 25 problems that involve pairs of consecutive letters. Each dataset is divided into three folds. The first fold is used for training, the second one for selecting the parameters $\gamma \! \in \! [0.01, 0.1] n_l$, \mbox{$\gamma_u \! \in \! [10^{-3}, 10^3] \gamma$, and} $\gamma_g = \gamma / \gamma_u$, and the last fold is used for testing.\footnote{Alternatively, the regularization parameters $\gamma$, $\gamma_u$, and $\gamma_g$ can be set using leave-one-out cross-validation on labeled examples.} The fraction of labeled examples in the training set is varied from 1 to 10 percent. All examples in the validation set are labeled and its size is limited to the number of labeled examples in the training set.

In all experiments, we use 5-nearest neighbor graphs whose edges are weighted as $w_{ij} = \exp[- \normw{\bx_i - \bx_j}{2}^2 / (2 K \sigma^2)]$, where $K$ is the number of features and $\sigma$ denotes the mean of their standard deviations. The width of radial basis functions (RBFs) is set accordingly to $\sqrt{K} \sigma$, and the threshold $\eps$ for choosing training examples (\ref{eq:MMGC}) is $10^{-6}$.

Test errors of all compared algorithms are averaged over all binary problems within each dataset and shown in Figure \ref{fig:UCI ML repository cuts}. Max-margin graph cuts outperform manifold regularization of SVMs in 29 out of 36 experiments. Note that the lowest errors are usually obtained for linear and cubic kernels, and our method improves the most over manifold regularization of SVMs in these settings.

\section{CONCLUSIONS}
\label{sec:conclusions}

This paper proposes a novel algorithm for semi-supervised learning. The algorithm learns max-margin graph cuts that are conditioned on the labels induced by the harmonic function solution. We motivate the approach, prove its generalization bound, and compare it to state-of-the-art algorithms for semi-supervised max-margin learning. The approach is evaluated on a synthetic problem and three UCI ML repository datasets, and we show that it usually outperforms manifold regularization of SVMs.

In our future work, we plan to investigate some of the shortcomings of this paper. For instance, note that the theoretical analysis of max-margin graph cuts (Section \ref{sec:theoretical analysis}) assumes soft labels but our solutions (\ref{eq:MMGC}) are computed using the hard labels $\ell_i = y_i$. Whether the theoretically sound setting yields better results in practice is an open question.

\bibliographystyle{theapa}
\bibliography{References}

\end{document}